\documentclass{svproc}
\usepackage{url}

\usepackage{algorithmicx}
\usepackage{setspace}
\usepackage{graphicx,psfrag,epsf}
\usepackage{booktabs}
\usepackage{algorithm}
\usepackage{algpseudocode}

\floatname{algorithm}{Algorithm}
\usepackage{url, amssymb, amsfonts}
\usepackage[colorlinks = true,
linkcolor = blue,
urlcolor  = blue,
citecolor = blue,
anchorcolor = blue]{hyperref}
\usepackage{graphicx}
\usepackage{amsmath}
\usepackage{paralist}
\usepackage{psfrag}
\usepackage[sort&compress,comma,square,numbers]{natbib}
\usepackage{url} 
\usepackage{setspace}
\usepackage{lscape}
\usepackage{color,amssymb}
\usepackage{dcolumn}
\usepackage{indentfirst, verbatim, float}

\usepackage{enumerate}

\begin{document}
\mainmatter              
\title{Fast and Scalable Multi-Kernel Encoder Classifier}
\titlerunning{Multi-Kernel Encoder Classifier}  
\author{Cencheng Shen}
\authorrunning{Cencheng Shen} 
\institute{University of Delaware, Newark, DE 19716, USA\\}

\maketitle              

\begin{abstract}
This paper introduces a new kernel-based classifier by viewing kernel matrices as generalized graphs and leveraging recent progress in graph embedding techniques. The proposed method facilitates fast and scalable kernel matrix embedding, and seamlessly integrates multiple kernels to enhance the learning process. Our theoretical analysis offers a population-level characterization of this approach using random variables. Empirically, our method demonstrates superior running time compared to standard approaches such as support vector machines and two-layer neural network, while achieving comparable classification accuracy across various simulated and real datasets.
\end{abstract}

\section{Introduction}
\label{sec:intro}
\noindent
The kernel-based methods have long been an important tool in machine learning. In the realm of supervised learning, the support vector machine (SVM) is a renowned method using kernels \citep{Cortes1995,ScholkopfSmolaMuller1999,Vapnik1999,Scholkopf2002learning}. The core objective of SVM is to maximize the margin of separation between distinct classes through the utilization of a kernel function, rendering it a suitable tool for classification. 

Assuming $\mathbf{Y} \in \{1,2\}^{n}$ and the corresponding sample data is $\mathbf{X} \in \mathbb{R}^{n \times p}$, let $\delta(\cdot,\cdot): \mathbb{R}^{p} \times \mathbb{R}^{p}$ be the kernel function, and $\mathbf{A} \in \mathbb{R}^{n \times n}$ be the sample kernel matrix, where
\begin{align*}
\mathbf{A}(i,j)=\delta(\mathbf{X}(i,:), \mathbf{X}(j,:)).
\end{align*}
SVM attempts to find a bounded and non-negative vector $\mathbf{C} \in \mathbb{R}^{1 \times n}$, such that
\begin{align*}
&\mathbf{C} = \arg\max_{\mathbf{C}} \frac{1}{2} (\mathbf{C}\otimes\mathbf{Y})\mathbf{A}(\mathbf{C}\otimes\mathbf{Y})^{T} - \sum_{i=1,\ldots,n}\mathbf{C}(i)\\
&\mbox{subject to } \mathbf{C}\mathbf{Y}^{T}=0.
\end{align*}
Excluding the intercept, SVM constructs a decision boundary based on
\begin{align*}
\mathbf{Z}=\mathbf{A} (\mathbf{C}\otimes\mathbf{Y}) \in \mathbb{R}^{n \times 1}.
\end{align*}

From a dimension reduction perspective, one can interpret $(\mathbf{C}\otimes\mathbf{Y})$ as a projection matrix, and $\mathbf{Z}$ as a one-dimensional representation of the kernel matrix $\mathbf{A}$. Although the steps involved are relatively straightforward, the computational bottlenecks of SVM primarily arise from kernel computation and solving the objective function. Additionally, in the case of multi-class classification where $\mathbf{Y} \in \{1,2,\ldots,K\}^{n}$, SVM requires the construction of $K(K-1)/2$ binary classifiers to separate each pair of classes, significantly slowing down the process for multi-class data. Furthermore, the kernel choices in SVM are typically limited to the inner product for a linear decision boundary or the Gaussian kernel for a nonlinear decision boundary. Ideally, one would like to compare different kernel choices, but that would require cross-validation, rendering it computationally intensive.

From an alternative perspective, a kernel matrix can be viewed as a similarity matrix or a weighted graph. Graph data has become increasingly prevalent in capturing relationships between entities, including but not limited to social networks, communication networks, webpage hyperlinks, and biological systems \cite{GirvanNewman2002, newman2003structure, barabasi2004network, boccaletti2006complex, VarchneyEtAl2011,DCorGraphScreening,GraphCorr}. Given $n$ vertices, a graph can be represented by an adjacency matrix $\mathbf{A} \in \mathbb{R}^{n \times n}$, which coincides with the kernel matrix.

As graph adjacency matrices are high-dimensional, an important approach to explore graph structure is through graph embedding, which aims to provide a low-dimensional Euclidean representation for each vertex. Some common methods include spectral embedding \cite{RoheEtAl2011,SussmanEtAl2012,Priebe2019}, node2vec \cite{grover2016node2vec, node2vec2021}, graph convolutional neural network \cite{kipf2017semi, Wu2019ACS, wang2022combining}, etc. To the best of our knowledge, there has been no existing investigation into using graph embedding approaches for kernel matrices. This is partly because most graph embedding approaches are relatively slow, and the theoretical properties usually focus on binary graphs, where entries of $\mathbf{A}$ are either $0$ or $1$. Notably, a recent graph-encoder embedding was proposed \citep{GEE1, GEEClustering}, which has linear complexity for graph data \citep{GEESparse, GEEParallel}, works well for weighted graphs \citep{GEEDistance}, and can be used for latent community recovery \citep{GEERefine} or in multiple-graph settings \citep{GEEFusion, GEEDynamics}.

Motivated by the concept of graph encoder embedding, this paper introduces a new approach for kernel-based classification. We begin by introducing an intermediate algorithm, which directly applies graph encoder embedding to a given kernel matrix. Subsequently, we present a faster version that streamlines matrix multiplication steps, reduces kernel computation complexity from $O(n^2)$ to $O(nK)$, and facilitates multi-kernel comparison through cross-entropy. From a theoretical standpoint, we explore the proposed method from a probabilistic perspective, offering insights into why and how it performs effectively. In the numerical experiments, we provide simulations and real-world examples using text and image data to demonstrate the performance of the proposed method, which achieves comparable classification accuracy in a fraction of the time compared to standard methods such as SVM and two-layer neural network. All proofs are in the appendix.

\section{Graph Encoder Embedding for Kernel}
\label{sec1}
\noindent
This intermediate algorithm directly applies graph encoder embedding to a given kernel matrix, followed by linear discriminant analysis.\\

\noindent
\textbf{Input}: The raw data $\mathbf{X} \in \mathbb{R}^{n \times p}$ and a label vector $\mathbf{Y} \in \{0,1,\ldots,K\}^{n}$, where values from $1$ to $K$ represent known labels, and $0$ for testing samples with unknown labels.

\noindent
\textbf{Step 1}: Form the $n \times n$ kernel matrix from $\mathbf{X}$, where
\begin{align*}
\mathbf{A}(i,j)= \delta(\mathbf{X}(i,:), \mathbf{X}(j,:)) \in \mathbb{R}^{n \times n}.
\end{align*}

\noindent
\textbf{Step 2}: Calculate the number of known observations per class, i.e., 
\begin{align}
\label{eq3}
n_k = \sum_{i=1}^{n} 1(\mathbf{Y}_i=k)
\end{align}
for $k=1,\ldots,K$. Additionally, identify the set of training indices with positive labels as $trn$.

\noindent
\textbf{Step 3}: Compute the matrix $\mathbf{W} \in [0,1]^{n \times K}$ as follow: for each vertex $i=1,\ldots,n$, set
\begin{align}
\label{eq4}
\mathbf{W}(i, k) = 1 / n_k
\end{align} 
if and only if $\mathbf{Y}_i=k$, and $0$ otherwise. Note that samples with unknown labels are effectively assigned zero values, meaning $\mathbf{W}(i, :)$ is a zero row-vector if and only if $\mathbf{Y}_i=0$.

\noindent
\textbf{Step 4}: Compute the kernel embedding:
\begin{align*}
\mathbf{Z}= \mathbf{A}\mathbf{W} \in \mathbb{R}^{n \times K}.
\end{align*}

\noindent
\textbf{Step 5}: Train a linear discriminant model $g(\cdot)$ on $(\mathbf{Z}(trn,:), \mathbf{Y}(trn))$.

\noindent
\textbf{Output}: The final embedding $\mathbf{Z}$ and the classification model $g(\cdot)$.
\\


\noindent
The algorithm is essentially the same as the one-hot graph encoder embedding in \cite{GEE1,GEEDistance}. The matrix $\mathbf{A}$ is a kernel matrix transformed from standard data, which can be viewed as a general graph structure. Note that the matrix $\mathbf{W}$ plays a similar role as the SVM transformation matrix $(\mathbf{C}\otimes\mathbf{Y}$). Specifically, the matrix $(\mathbf{C}\otimes\mathbf{Y})$ assign $0$ weights to samples out of the margin, and similarly $\mathbf{W}$ only assigns weights in the dimension same as the sample classes. Subsequently, both matrices are multiplied with $\mathbf{A}$.

Nevertheless, $\mathbf{W}$ offers distinct advantages. It is considerably faster to construct than solving the SVM objective function, creates margin of separation in a multivariate space, and adeptly handles multi-class (i.e., $K>2$) data without necessitating pairwise classification models as in SVM. It is important to emphasize that testing labels were not utilized in the embedding process; instead, the testing data were embedded using the same $\mathbf{W}$ constructed exclusively from the training labels.

From a time complexity perspective, the kernel matrix computation in Step 1 costs $O(n^2 p)$. Step 2 and 3 construct the embedding matrix $\mathbf{W}$ directly using training labels, which has a complexity of $O(n)$. Step 4 is a standard matrix multiplication and costs $O(n^2)$. Assuming a fixed constant $K$, linear discriminant analysis in Step 5 costs $O(n)$. Consequently, this intermediate kernel encoder classifier has a time complexity of $O(n^2 p)$.

\section{Fast Multi-Kernel Encoder Classifier}
\label{sec2}
\noindent
Building upon the intermediate algorithm, in this section we propose an optimized multi-kernel encoder classifier. This optimized approach simplifies matrix multiplication and enables the selection of the optimal kernel from multiple choices through the utilization of cross-entropy.\\

\noindent
\textbf{Input}: The given data $(\mathbf{X},\mathbf{Y})$, and a set of $M$ different kernel choices: 
\begin{align*}
\{\delta_1(\cdot,\cdot), \cdots, \delta_M(\cdot,\cdot)\}.
\end{align*} 

\noindent
\textbf{Step 1}: Compute the number of known observations per class $n_k$, same as in Equation~\ref{eq3}. Additionally, identify the set of training indices with positive labels as $trn$.

\noindent
\textbf{Step 2}: Compute the matrix $\mathbf{W} \in [0,1]^{n \times K}$, same as in Equation~\ref{eq4}. Furthermore, the one-hot encoding matrix $\mathbf{V}$ of the label vector is also calculated:
\begin{align*}
\mathbf{V}(i, k) = 1
\end{align*} 
if and only if $\mathbf{Y}_i=k$, and $0$ otherwise. Note that $\mathbf{V}(i, :)$ is a zero row-vector if and only if $\mathbf{Y}_i=0$, i.e., labels with zero values are not used.

\noindent
\textbf{Step 3}: Compute the matrix $\mathbf{U} \in \mathbb{R}^{K \times p}$ as 
\begin{align*}
\mathbf{U} = \mathbf{W}^{T} \mathbf{X}
\end{align*} 
Then, for each kernel choice $\delta_m(\cdot,\cdot)$, compute the embedding $\mathbf{Z}^{\delta_{m}}$ as follows:
\begin{align*}
\mathbf{Z}^{\delta_{m}}(i,j)= \delta_{m}(\mathbf{X}(i,:), \mathbf{U}(j,:)) \in \mathbb{R}^{n \times K}
\end{align*}
for $i =1,\ldots,n$ and $j=1,\ldots, K$.

\noindent
\textbf{Step 4}: For each $m$, a linear discriminant model $g_{m}(\cdot)$ is trained on $(\mathbf{Z}^{\delta_{m}}(trn,:), \mathbf{Y}(trn))$, and denote 
\begin{align*}
\mathbf{T}_{m}= g_{m}(\mathbf{Z}^{\delta_{m}}) \in \mathbb{R}^{n \times K}.
\end{align*}
Then, compute the cross-entropy with respect to $\mathbf{V}$:
\begin{align*}
c_{m}= -\sum\limits_{k=1}^{K}\sum\limits_{i=1}^{n} \mathbf{V}(i,k) \log(\mathbf{T}_{m}(i,k))
\end{align*}
for $m=1,\ldots,M$.

\noindent
\textbf{Step 5}: Select the optimal kernel choice that minimizes cross-entropy, i.e., 
\begin{align*}
m^{*}= \arg\min_{m=1,\ldots,M} c_{m}.
\end{align*}

\noindent
\textbf{Output}: The best multi-kernel embedding $\mathbf{Z}_{m^{*}}$ and the classification model $g_{m^{*}}(\cdot)$.
\\

\noindent
The optimized method utilizes the matrix $\mathbf{U}$ as a proxy and only calculates the pairwise kernels between $\mathbf{X}$ and $\mathbf{U}$, thereby eliminating the need to compute the full kernel matrix. This improvement substantially reduces the computational burden, and makes multi-kernel comparison more practical. Given $M$ different kernels, both SVM and the intermediate method in Section~\ref{sec1} would require $O(Mn^2p)$ in terms of time complexity. In contrast, the optimized method only needs $O(MKnp)$, resulting in linear time complexity with low constant overhead. 

Another key insight lies in the fact that the resulting embedding readily facilitates the comparison of kernel choices. As the final embedding is in dimension $K$ and is properly normalized through linear discriminant analysis (meaning that a large value in dimension $k$ indicates class membership in class $k$), it can be directly used in cross-entropy computation, where a smaller value indicates better class separation. In our experiments, to ensure numerical scalability, we utilize the cross-entropy from the inner product as a benchmark and only switch to another kernel if the resulting cross-entropy is at least $30\%$ smaller.

It is worth noting that in the case of $M=1$ and the kernel function being the inner product, the optimized method is identical to the intermediate method. In other words, the final embedding and the classifier coincide. For general kernel functions, the resulting classifier may be different, but they still approximate each other, and the numerical results are very similar for all the kernel choices we experimented with. Further explanations regarding $\mathbf{U}$ and why it provides a proper approximation are provided in the next section from a probabilistic perspective.

\section{Population Theory}
\noindent
In this section, we consider the proposed method within the standard probabilistic framework to gain a better understanding of its behavior. We begin with the following standard probabilistic assumptions for the input data. 

Let $K$ and $p$ be fixed, and consider the random variables $(X, Y) \in \mathbb{R}^{p} \times [K]$. Here, $Y$ follows a categorical distribution with prior probabilities:
\begin{align*}
{\pi_k \in (0,1], \sum_{k=1}^{K} \pi_k = 1}, 
\end{align*}
and $X$ has finite moments and follows a K-component mixture distribution:
\begin{align*}
X \sim \sum_{k=1}^{K} \pi_k f_{X|Y}(a_j), j=1,\ldots,m.
\end{align*}
Then, each sample data point is assumed to be independently and identically distributed:
\begin{align*}
(\mathbf{X}(i,:), \mathbf{Y}(i)) \stackrel{i.i.d.}{\sim} (X,Y)
\end{align*}
for each $i=1,\ldots,n$. 

In the probabilistic framework, the transformation matrix $\mathbf{U}$ in Section~\ref{sec2} can be understood as follows:
\begin{theorem}
\label{thm1}
Given the random variable pair $(X,Y)$ of finite moments. Let $U \in \mathbb{R}^{K \times p}$, where each row satisfies
\begin{align*}
U(k,:) &= E(X|Y=k).
\end{align*}

Under the above probabilistic assumption for the sample data $(\mathbf{X}, \mathbf{Y})$, the matrix $\mathbf{U}$ in Section~\ref{sec2} satisfies:
\begin{align*}
\|\mathbf{U}- U \|_{F} \stackrel{n\rightarrow \infty}{\rightarrow} 0
\end{align*}
for the Frobenius matrix norm.
\end{theorem}
\noindent
Therefore, the transformation matrix $\mathbf{U}$ can be interpreted as a sample estimate of the matrix $U$, which is a fixed matrix consisting of class-wise conditional expectation. Consequently, given a single kernel function $\delta(\cdot,\cdot)$, the sample embedding
$\mathbf{Z}(i,j)=\delta(\mathbf{X}(i,:), \mathbf{U}(j,:))$ can be considered as a sample realization of the random variable $Z \in \mathbb{R}^{K}$, where 
\begin{align*}
Z_{j}=\delta(X, U(j,:))
\end{align*}
for $j=1,\ldots,K$. This characterization explains the main difference between the proposed algorithm in Section~\ref{sec2} and the intermediate version in Section~\ref{sec1}: instead of computing the pairwise kernel matrices for all samples of $X$, the fast algorithm only computes the kernel between samples of $X$ and the estimate of $U$.

Now, if we consider the special case where the kernel function $\delta(\cdot, \cdot)$ is simply inner product, then the embedding random variable satisfies
\begin{align*}
Z=X U^{T}.
\end{align*}
As $U$ is a fixed matrix, if the random variable $X$ is normally distributed per class, the resulting encoder embedding is also normally distributed. In this scenario, linear discriminant analysis is an ideal choice, because it not only normalizes the embedding variable $Z$ but is also the optimal classifier for the joint distribution of $(Z,Y)$. 

\begin{theorem}
\label{thm2}
Suppose $X$ is a mixture of normal distribution, with each normal component having the same variance matrix, i.e.,
\begin{align*}
X|(Y=y) \sim N(\mu_{y}, \Sigma).
\end{align*}
Then the embedding variable $Z=X U^{T}$ satisfies:
\begin{align*}
Z|(Y=y) \sim N(\mu_{y}U^{T}, U\Sigma U^{T}),
\end{align*}
in which case the linear discriminant analysis is the Bayes optimal classifier for $(Z,Y)$.
\end{theorem}
\noindent
The above theorem provides rationale for using linear discriminant analysis on the embedding $\mathbf{Z}$. While the assumption of a shared variance matrix might appear restrictive, it simplifies the covariance estimation step and often yields better results compared to quadratic discriminant analysis, which assumes different components have distinct variance matrices.

Finally, when using the inner product kernel in the proposed method, the resulting encoder embedding variable $Z$ can be viewed as an effective dimension reduction technique that preserves the L2 norm margin between the conditional expectations. This is supported by the following theorem:
\begin{theorem}
\label{thm3}
Suppose that the conditional expectations in the original space $\mathbb{R}^{p}$ are separated by certain margin, i.e.
\begin{align*}
\|E(X | Y=k) - E(X | Y=j)\| = \alpha_{kj}
\end{align*}
for any $k \neq j$.

When using the inner product kernel, the resulting encoder embedding variable $Z$ in the reduced $\mathbb{R}^{K}$ space satisfies
\begin{align*}
\|E(Z | Y=k) - E(Z | Y=j)\| \geq \frac{1}{\sqrt{2}} \alpha_{kj}^2.
\end{align*}
\end{theorem}
\noindent
For instance, consider an ideal scenario where the variance of the random variable $X$ is sufficiently small, and the conditional means of $X$ are well-separated, leading to perfect separability of different components of $X$ in the original $p$-dimensional space. In such cases, the encoder embedding using the inner product shall also exhibit perfect separability in the reduced $K$-dimensional space.

This theorem also highlights a distinction between the proposed method and SVM, despite both utilizing kernel transformations. Theoretically, SVM is designed to identify the optimal margin of separation for a given dataset by transforming it into a high-dimensional kernel space and solving an objective function that maximizes this margin. 

On the other hand, our method does not aim to find the best margin of separation. Instead, it constructs a fast transformation via the training data and the kernel function. The theorem demonstrates that this transformation from the original data space into a $K$-dimensional space (via the kernel space) readily preserves the margin of separation. In high-dimensional data, estimating the margin of separation can be challenging. As the proposed kernel-based algorithm effectively preserves this margin in a lower-dimensional subspace, it allows subsequent methods (such as linear discriminant analysis in our case) to better estimate this margin of separation. 

It is worth noting that the margin may actually be improved, as $\alpha^2/\sqrt{2}$ is larger than $\alpha$ when $\alpha > \sqrt{2}$. However, the extent of this improvement and whether it actually benefits the eventual classification depend on the actual data distribution, making it a topic of future investigation. Since there is no explicit objective function involved to enlarge the margin, the proposed approach trades potential accuracy gains for better computational speed. Our simulations and real data experiments indicate that the proposed method can perform slightly better or worse than SVM, depending on the specific dataset, but it consistently offers significantly faster computation.

\section{Experiments}
\noindent
Throughout both the simulations and real data experiments, we considered four distinct methods for comparison: the linear encoder classifier, the multi-kernel encoder classifier, support vector machine (SVM), and a two-layer neural network. In each simulation, we conducted 5-fold cross-validation, repeated the process for 20 different Monte-Carlo replicates, and reported the average classification error and running time. All experiments were performed using MATLAB 2022a on a standard desktop equipped with a 16-core Intel CPU and 64GB memory.

The linear encoder classifier employs the algorithm in Section~\ref{sec2} with the inner product as its sole kernel choice. Te multi-kernel encoder classifier utilizes the same algorithm with three distinct kernel choices: the inner product, the Euclidean distance-induced kernel (based on the distance-to-kernel transformation in \citep{DcorKernel}), and the Spearman rank correlation \citep{KendallBook}. These alternative kernels have been empirically shown to perform effectively and can capture nonlinear decision boundaries.

For the support vector machine, we utilized MATLAB's fitcecoc\footnote{\url{https://www.mathworks.com/help/stats/fitcecoc.html}} method to fit multiclass models \citep{Allwein2000,Escalera2009}. All parameters were set to their default values, with SVM using the Gaussian kernel. The two-layer neural network used MATLAB's fitcnet\footnote{\url{https://www.mathworks.com/help/stats/fitcnet.html}} method, with a neuron size of $100$ and all other parameters at their default settings \citep{Nocedal2006,Glorot2010}.

\subsection{Simulations}
\noindent
The simulations in this study involve the following six settings, with $K=5$, $Y=\{1,2,\ldots,K\}$ equally likely, $p=5000$, and $n$ ranging from $50$ to $500$:
\begin{itemize}
\item High-Dimensional Uniform: $X_{Y} \sim Uniform(1,3)$, and $X_{i} \sim Uniform(0,1)$ for every $i \neq y$.
\item Uniform + Noise: Same as above, then add uniform noise to each dimension, i.e., $X_i=X_i+0.5*Uniform(0,1)$ for each $i$.
\item Uniform Transformed: Same as above, followed by an additional random transformation, i.e., $X =X *Q$, where $Q \in \mathbb{R}^{p \times p}$ and each entry of $Q$ is randomly generated from $Uniform(0,1)$ during each replicate.
\item High-Dimensional Normal: $X_{Y} \sim Normal(8,1)$, and $X_{i} \sim Normal(1,1)$ for every $i \neq y$.
\item Normal + Noise: Same as above, then add uniform noise to each dimension, i.e., $X_i=X_i+2*Uniform(0,1)$ for each $i$.
\item Normal Transformed: Same as above, followed by an additional random transformation, i.e., $X =X *Q$, where $Q \in \mathbb{R}^{p \times p}$ and each entry of $Q$ is randomly generated from $Uniform(0,1)$ during each replicate.
\end{itemize}
These simulations are high-dimensional, with the noisy case being more challenging and the transformed case being the most difficult. Nevertheless, the signal to separate the classes is concentrated within a $5$-dimensional subspace, specifically in the top $K$ dimensions within the non-transformed models, and in a random subspace in the dimension-transformed scenarios. In all these simulations, the optimal classification error equals to $0$. This optimal error can be determined by computing the Bayes error based on the probability model.

In each Monte-Carlo replicate, $(\mathbf{X},\mathbf{Y})$ is generated independently from the specified distribution. Evaluation of each method is performed using the same 5-fold split, and the results for average classification error and running time are presented in Figure~\ref{fig0}.

Across all simulations, it is observed that the linear encoder, multi-kernel encoder, and SVM perform similarly well in terms of classification error, all converging to $0$ error as the sample size increases. The two-layer neural network shows slightly worse performance in the non-transformed simulations and significantly worse performance in the dimension-transformed simulations. In terms of running time, the linear encoder is the most efficient, followed by the multi-encoder. 

\begin{figure*}[htbp]
\centering
	\includegraphics[width=0.99\textwidth,trim={0cm 0cm 0cm 0cm},clip]{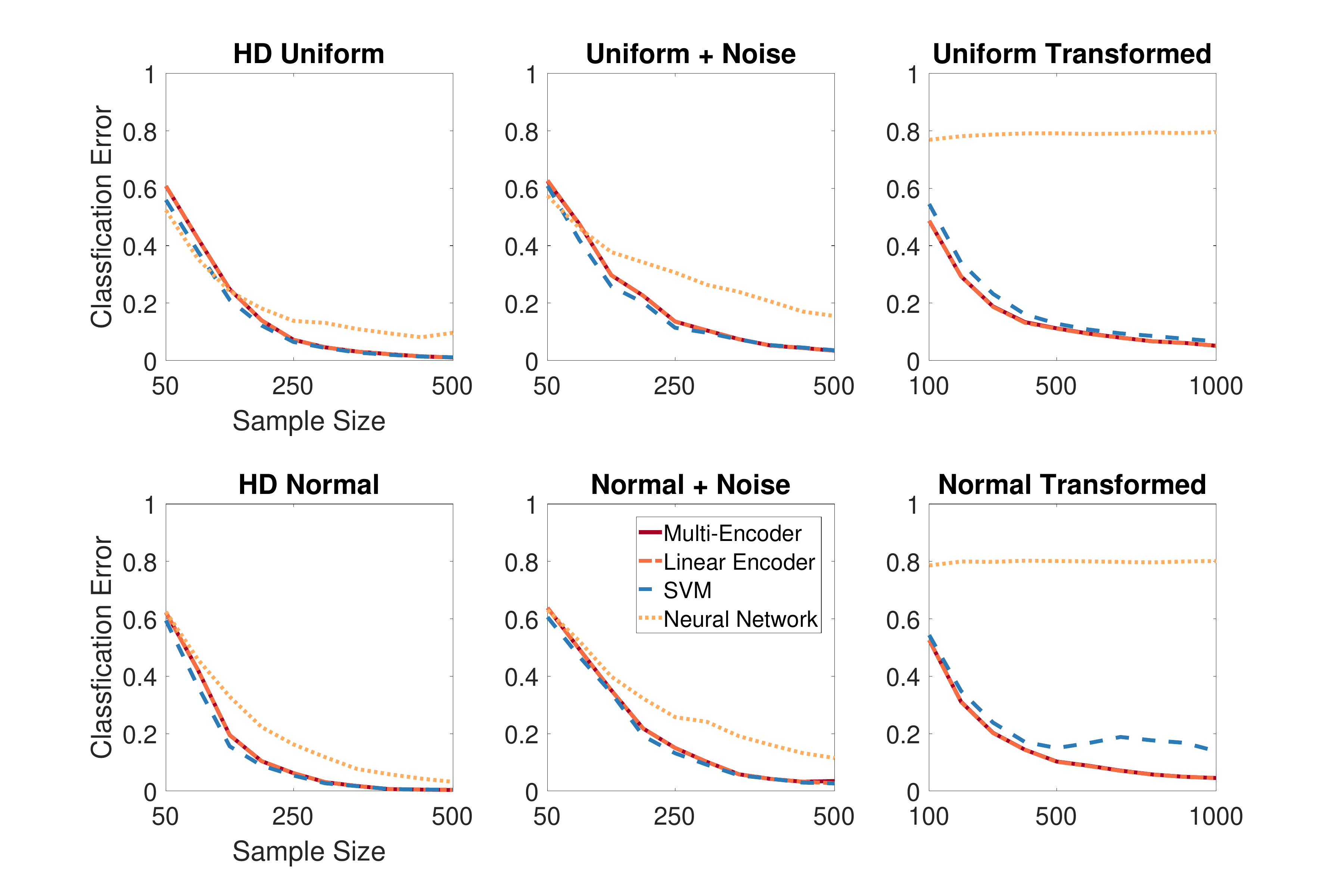}
	\includegraphics[width=0.99\textwidth,trim={0cm 0cm 0cm 0cm},clip]{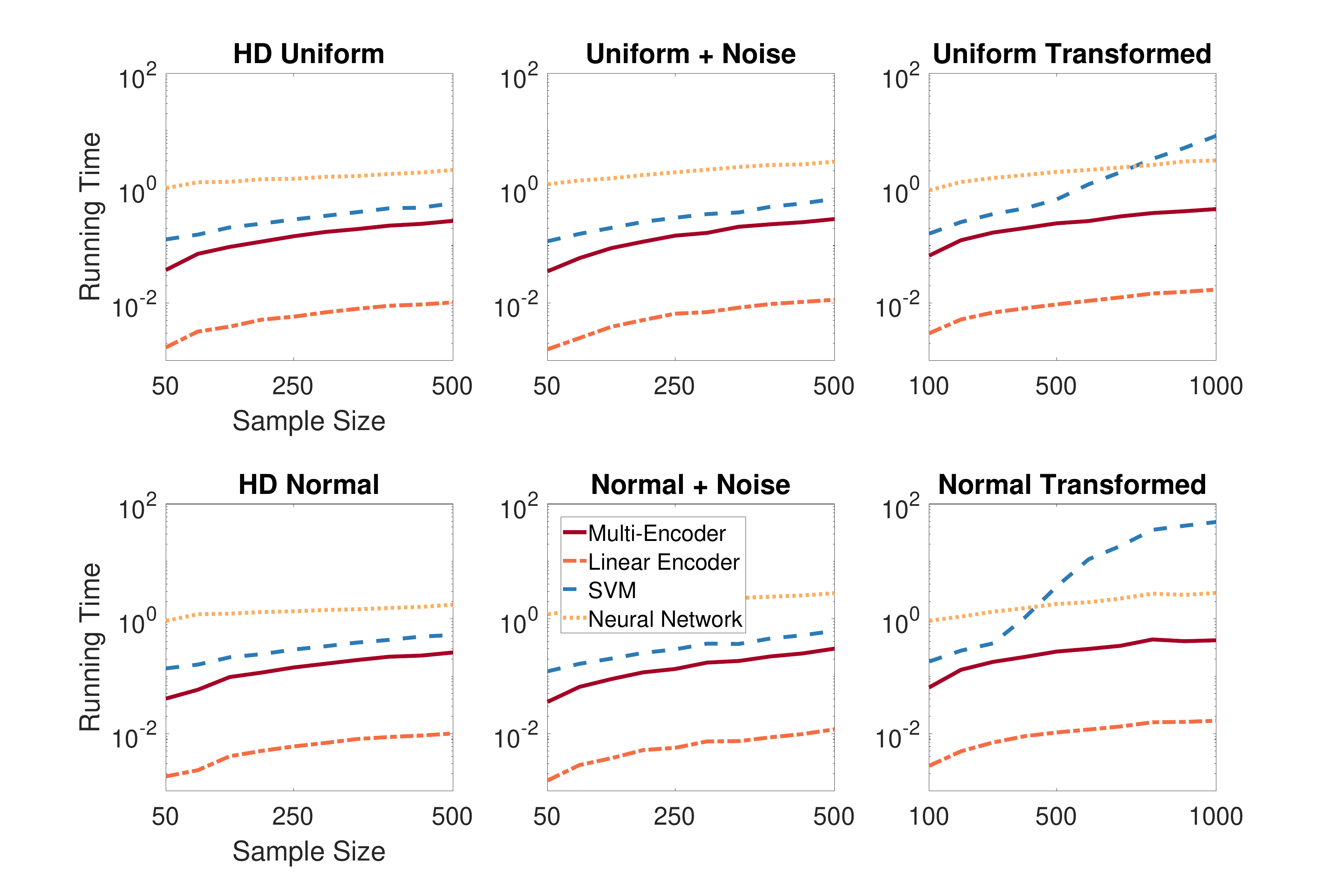}
	\caption{The figure provides a comparison of classification errors and running times across six different simulations. These comparisons were conducted utilizing 20 Monte-Carlo replicates and a 5-fold evaluation.}
	\label{fig0}
\end{figure*}

\subsection{Real Data}
\noindent
In this section, we assess the classification performance using a diverse range of real-world datasets, including three text datasets and six face image datasets, each with its own characteristics:
\begin{itemize}
\item Cora dataset \citep{mccallum2000automating} is a citation network with $2708$ papers and $7$ classes. Each paper is represented by a $1433$-dimensional binary vector indicating the presence or absence of corresponding words from a dictionary.
\item Citeseer data \citep{giles1998citeseer} is another citation network with $3312$ papers and $6$ classes, featuring attribute vectors of $3703$ dimensions for word presence.
\item ISOLET spoken data language \citep{cole1990spoken} is a database of spoken English letters, consisting of $7797$ spoken letters. The data is represented in a $617$-dimensional space, and there are $26$ classes.
\item ORL \citep{SamariaHarter1994} has $400$ face images of $40$ subjects. ORL 1 images are of size $32 \times 32$ pixels, while ORL 2 images are of size $64 \times 64$ pixels, resulting in dimensionality of $1024$ and $4096$, respectively.
\item CMU PIE \citep{SimBakerBsat2003,HeEtAl2005} includes $11554$ face images from $68$ subjects, each sized at $32 \times 32$ pixels.
\item Extended Yale B database \citep{GeorphiadesBelheumeurKriegman2001, LeeHoKriegman2005} contains $2414$ face images of $38$ individuals under various poses and lighting conditions. The images are resized to $32 \times 32$ pixels. Two other versions, Yale 1 and 2, are also experimented, which consist of $15$ subjects, $165$ images, and are sized at $32 \times 32$ and $64 \times 64$ pixels, respectively.
\end{itemize}
Table~\ref{table1} reports the classification error and running time of each method. We make the following observations:

The linear encoder method consistently demonstrates exceptional classification performance across all datasets, often ranking as the best or among the best methods. Notably, it also stands out for its speed, as it exhibits nearly instant processing for the given datasets.

The multi-kernel encoder, while slightly more time-consuming than the linear encoder, remains significantly faster than other competing methods. Moreover, its ability to evaluate two additional kernel functions enhances its adaptability to the data, resulting in improved classification performance for specific datasets. For example, it selects the Euclidean distance-induced kernel for the CMU PIE dataset and the Spearman correlation for the Extended Yale dataset, leading to substantial reductions in classification errors compared to the linear encoder.

In comparison, SVM and neural network consume considerably more time. Between them, the support vector machine is the most competitive, with classification performance similar to that of the linear encoder and multi-kernel encoder, in line with the findings from the simulations. However, SVM is also the most computationally expensive method, especially when dealing with image datasets featuring large dimensions and moderate numbers of classes. 

\begin{table*}[htbp]
\renewcommand{\arraystretch}{1.3}
\centering
\caption{This table displays the 5-fold classification error and running time for each method. It includes both the average and the standard deviation, with the best error highlighted in bold within each dataset.}
\scalebox{1.0}{
\begin{tabular}{|c||c|c|c|c|c}
\hline
5-Fold Error & Multi-Encoder & Linear Encoder & SVM & Two-Layer NN\\
 \hline
Citeseer  & $29.8\%\pm0.2\%$ & $29.8\%\pm0.2\%$   & $\textbf{28.7}\% \pm 0.4\%$ & $30.9\% \pm 0.4\%$ \\
Cora & $28.8\%\pm0.6\%$ & $28.9\%\pm0.2\%$   & $25.5\% \pm 0.4\%$ & $\textbf{25.1}\% \pm 0.4\%$ \\
ISOLET & $7.6\%\pm0.1\%$ & $7.6\%\pm0.1\%$   & $\textbf{3.6}\% \pm 0.1\%$ & $4.2\% \pm 0.1\%$ \\
ORL1  & $\textbf{2.0\%}\pm0.4\%$ & $\textbf{2.0\%}\pm0.4\%$   & $4.8\% \pm 0.3\%$ & $78.9\% \pm 7.2\%$ \\
ORL2  & $\textbf{2.6\%}\pm0.4\%$ & $\textbf{2.6\%}\pm0.4\%$   & $5.1\% \pm 0.4\%$ & $84.6\% \pm 5.0\%$ \\
PIE & $4.7\%\pm0.1\%$ & $8.9\%\pm0.1\%$   & $\textbf{2.5\%} \pm 0.1\%$ & $33.7\% \pm 7.5\%$ \\
Yale Extended & $\textbf{0.9\%}\pm0.1\%$ & $17.0\%\pm0.3\%$   & $9.7\% \pm 0.6\%$  & $18.6\% \pm 0.4\%$ \\
Yale1  & $20.4\%\pm1.5\%$ & $\textbf{20.0\%}\pm1.4\%$   & $26.8\% \pm 1.1\%$ & $59.2\% \pm 0.6\%$ \\
Yale2 & $\textbf{13.6\%}\pm1.4\%$ & $\textbf{13.6\%}\pm1.4\%$   & $20.1\% \pm 0.9\%$ & $52.1\% \pm 7.8\%$ \\
\hline
 \hline
 Running Time &  Multi-Encoder & Linear Encoder & SVM & Two-Layer NN\\
\hline
Citeseer  & $1.25\pm0.13$   & $\textbf{0.05}\pm0.01$ & $8.45\pm0.24$ & $4.53\pm0.17$\\
Cora & $0.43\pm0.05$   & $\textbf{0.01}\pm0.00$ & $1.39\pm0.06$& $1.56\pm0.06$\\
ISOLET & $1.10\pm0.01$   & $\textbf{0.03}\pm0.00$ & $7.95\pm0.23$ & $6.58\pm0.29$\\
ORL1  & $0.13\pm0.05$   & $\textbf{0.01}\pm0.00$ & $32.7\pm2.6$ & $6.87\pm2.32$\\
ORL2  & $0.36\pm0.03$   & $\textbf{0.01}\pm0.00$ & $139\pm7.3$ & $24.0\pm7.3$\\
PIE & $3.21\pm0.07$   & $\textbf{0.09}\pm0.01$ & $1410\pm35.1$& $112\pm8.1$\\
Yale Extended & $0.69\pm0.02$   & $\textbf{0.01}\pm0.00$ & $353\pm7.9$ & $26.3\pm2.1$\\
Yale1 & $0.06\pm0.01$   & $\textbf{0.01}\pm0.00$ & $12.6\pm1.1$ & $1.22\pm0.91$\\
Yale2 & $0.14\pm0.02$   & $\textbf{0.01}\pm0.00$ & $29.9\pm1.8$ & $4.99\pm2.4$\\
\hline
\end{tabular}
}
\label{table1}
\end{table*}

\section{Conclusion}
\noindent
This paper introduces a new classifier based on kernels and graph embedding. The experimental results presented confirm the advantages of the proposed algorithm, demonstrating excellent classification accuracy comparable to that of SVM. Notably, it achieves this performance with significantly improved runtime, often instantaneous for moderately-sized multivariate data. While other variants of SVM and advanced architectures of neural networks may likely achieve better accuracy on benchmark data, the proposed approach offers a competitive alternative to standard baseline classifiers, enabling rapid classification with arbitrary kernel choices and exceptional scalability for big data analytics.

\section*{Acknowledgement}
\addcontentsline{toc}{section}{Acknowledgment}
\noindent
This work was supported by the National Science Foundation DMS-2113099, and by funding from Microsoft Research.

\bibliographystyle{spmpsci}
\bibliography{spmpsci}

\begin{thebibliography}{10}
\providecommand{\url}[1]{{#1}}
\providecommand{\urlprefix}{URL }
\expandafter\ifx\csname urlstyle\endcsname\relax
  \providecommand{\doi}[1]{DOI~\discretionary{}{}{}#1}\else
  \providecommand{\doi}{DOI~\discretionary{}{}{}\begingroup \urlstyle{rm}\Url}\fi

\bibitem{Allwein2000}
Allwein, E., Schapire, R., Singer, Y.: Reducing multiclass to binary: A unifying approach for margin classifiers.
\newblock Journal of Machine Learning Research \textbf{1}, 113--141 (2000)

\bibitem{barabasi2004network}
Barabási, A.L., Oltvai, Z.N.: Network biology: Understanding the cell's functional organization.
\newblock Nature Reviews Genetics \textbf{5}(2), 101--113 (2004)

\bibitem{boccaletti2006complex}
Boccaletti, S., Latora, V., Moreno, Y., Chavez, M., Hwang, D.U.: Complex networks: Structure and dynamics.
\newblock Physics Reports \textbf{424}(4-5), 175--308 (2006)

\bibitem{cole1990spoken}
Cole, R., Fanty, M.: Spoken letter recognition.
\newblock In: Proc. Third DARPA Speech and Natural Language Workshop (1990)

\bibitem{Cortes1995}
Cortes, C., Vapnik, V.: Support-vector networks.
\newblock Machine Learning \textbf{20}(3), 273--297 (1995)

\bibitem{DevroyeGyorfiLugosiBook}
Devroye, L., Gyorfi, L., Lugosi, G.: A Probabilistic Theory of Pattern Recognition.
\newblock Springer (1996)

\bibitem{Escalera2009}
Escalera, S., Pujol, O., Radeva, P.: Separability of ternary codes for sparse designs of error-correcting output codes.
\newblock Pattern Recog. Lett. \textbf{30}(3), 285--297 (2009)

\bibitem{GeorphiadesBelheumeurKriegman2001}
Georghiades, A., Buelhumeur, P., Kriegman, D.: From few to many: Illumination cone models for face recognition under variable lighting and pose.
\newblock IEEE Transactions on Pattern Analysis and Machine Intelligence \textbf{23}(6), 643--660 (2001)

\bibitem{giles1998citeseer}
Giles, C.L., Bollacker, K.D., Lawrence, S.: Citeseer: An automatic citation indexing system.
\newblock In: Proceedings of the Third ACM Conference on Digital Libraries, pp. 89--98 (1998)

\bibitem{GirvanNewman2002}
Girvan, M., Newman, M.E.J.: Community structure in social and biological networks.
\newblock Proceedings of National Academy of Science \textbf{99}(12), 7821--7826 (2002)

\bibitem{Glorot2010}
Glorot, X., Bengio, Y.: Understanding the difficulty of training deep feedforward neural networks.
\newblock In: Proceedings of the thirteenth international conference on artificial intelligence and statistics, pp. 249--256 (2010)

\bibitem{grover2016node2vec}
Grover, A., Leskovec, J.: node2vec: Scalable feature learning for networks.
\newblock In: Proceedings of The 22nd ACM SIGKDD international conference on Knowledge discovery and data mining, pp. 855--864 (2016)

\bibitem{HeEtAl2005}
He, X., Yan, S., Hu, Y., Niyogi, P., Zhang, H.: Face recognition using {Laplacianfaces}.
\newblock IEEE Transactions on Pattern Analysis and Machine Intelligence \textbf{27}(3), 328--340 (2005)

\bibitem{KendallBook}
Kendall, M.G.: Rank Correlation Methods.
\newblock London: Griffin (1970)

\bibitem{kipf2017semi}
Kipf, T.N., Welling, M.: Semi-supervised classification with graph convolutional networks.
\newblock In: International Conference on Learning Representations (2017)

\bibitem{LeeHoKriegman2005}
Lee, K., Ho, J., Kriegman, D.: Acquiring linear subspaces for face recognition under variable lighting.
\newblock IEEE Transactions on Pattern Analysis and Machine Intelligence \textbf{27}(5), 684--698 (2005)

\bibitem{node2vec2021}
Liu, R., Krishnan, A.: Pecanpy: a fast, efficient and parallelized python implementation of node2vec.
\newblock Bioinformatics \textbf{37}(19), 3377--3379 (2021)

\bibitem{GEEParallel}
Lubonja, A., Shen, C., Priebe, C.E., Burns, R.: Edge-parallel graph encoder embedding.
\newblock In: 2024 38th IEEE International Parallel and Distributed Processing Symposium, Workshop on Graphs, Architectures, Programming, and Learning (2024)

\bibitem{mccallum2000automating}
McCallum, A.K., Nigam, K., Rennie, J., Seymore, K.: Automating the construction of internet portals with machine learning.
\newblock Information Retrieval \textbf{3}, 127--163 (2000)

\bibitem{newman2003structure}
Newman, M.E.J.: The structure and function of complex networks.
\newblock SIAM Review \textbf{45}(2), 167--256 (2003)

\bibitem{Nocedal2006}
Nocedal, J., Wright, S.J.: Numerical Optimization, 2nd edn.
\newblock Springer, New York (2006)

\bibitem{Priebe2019}
Priebe, C., Parker, Y., Vogelstein, J., Conroy, J., Lyzinskic, V., Tang, M., Athreya, A., Cape, J., Bridgeford, E.: On a 'two truths' phenomenon in spectral graph clustering.
\newblock Proceedings of the National Academy of Sciences \textbf{116}(13), 5995--5600 (2019)

\bibitem{GEESparse}
Qin, X., Shen, C.: Efficient graph encoder embedding for large sparse graphs in python.
\newblock In: Intelligent Computing 2024, vol.~3, pp. 568--577. Springer Nature (2024)

\bibitem{RoheEtAl2011}
Rohe, K., Chatterjee, S., Yu, B.: Spectral clustering and the high-dimensional stochastic blockmodel.
\newblock Annals of Statistics \textbf{39}(4), 1878--1915 (2011)

\bibitem{SamariaHarter1994}
Samaria, F., Harter, A.: Parameterisation of a stochastic model for human face identification.
\newblock In: Proceedings of the Second IEEE Workshop on Applications of Computer Vision, pp. 138--142 (1994)

\bibitem{Scholkopf2002learning}
Sch{\"o}lkopf, B., Smola, A.J.: Learning with kernels: support vector machines, regularization, optimization, and beyond.
\newblock MIT press (2002)

\bibitem{ScholkopfSmolaMuller1999}
Schölkopf, B., Smola, A., Müller, K.: Kernel principal component analysis.
\newblock In: Advances in kernel methods - support vector learning, pp. 327--352. MIT Press (1999)

\bibitem{GEEDistance}
Shen, C.: Encoder embedding for general graph and node classification.
\newblock arXiv preprint arXiv:2405.15473  (2024)

\bibitem{GraphCorr}
Shen, C., Arroyo, J., Xiong, J., Vogelstein, J.T.: Community correlations and testing independence between binary graphs.
\newblock arXiv preprint arXiv:1906.03661  (2024)

\bibitem{GEERefine}
Shen, C., Larson, J., Trinh, H., Priebe, C.E.: Refined graph encoder embedding via self-training and latent community recovery.
\newblock arXiv preprint arXiv:2405.12797  (2024)

\bibitem{GEEDynamics}
Shen, C., Larson, J., Trinh, H., Qin, X., Park, Y., Priebe, C.E.: Discovering communication pattern shifts in large-scale labeled networks using encoder embedding and vertex dynamics.
\newblock IEEE Transactions on Network Science and Engineering \textbf{11}(2), 2100--2109 (2024)

\bibitem{GEEClustering}
Shen, C., Park, Y., Priebe, C.E.: Graph encoder ensemble for simultaneous vertex embedding and community detection.
\newblock In: 2023 2nd International Conference on Algorithms, Data Mining, and Information Technology. ACM (2023)

\bibitem{GEEFusion}
Shen, C., Priebe, C.E., Larson, J., Trinh, H.: Synergistic graph fusion via encoder embedding.
\newblock Information Sciences \textbf{678}, 120,912 (2024)

\bibitem{DcorKernel}
Shen, C., Vogelstein, J.T.: The exact equivalence of distance and kernel methods in hypothesis testing.
\newblock AStA Advances in Statistical Analysis \textbf{105}(3), 385--403 (2021)

\bibitem{GEE1}
Shen, C., Wang, Q., Priebe, C.E.: One-hot graph encoder embedding.
\newblock IEEE Transactions on Pattern Analysis and Machine Intelligence \textbf{45}(6), 7933--7938 (2023)

\bibitem{DCorGraphScreening}
Shen, C., Wang, S., Badea, A., Priebe, C.E., Vogelstein, J.T.: Discovering the signal subgraph: An iterative screening approach on graphs.
\newblock Pattern Recognition Letters \textbf{184}, 97--102 (2024)

\bibitem{SimBakerBsat2003}
Sim, T., Baker, S., Bsat, M.: The {CMU} pose, illumination, and expression database.
\newblock IEEE Transactions on Pattern Analysis and Machine Intelligence \textbf{25}(12), 1615--1618 (2003)

\bibitem{SussmanEtAl2012}
Sussman, D., Tang, M., Fishkind, D., Priebe, C.: A consistent adjacency spectral embedding for stochastic blockmodel graphs.
\newblock Journal of the American Statistical Association \textbf{107}(499), 1119--1128 (2012)

\bibitem{Vapnik1999}
Vapnik, V.: The Nature of Statistical Learning Theory, 2nd edn.
\newblock Springer, New York (1999)

\bibitem{VarchneyEtAl2011}
Varshney, L., Chen, B., Paniagua, E., Hall, D., Chklovskii, D.: Structural properties of the caenorhabditis elegans neuronal network.
\newblock PLoS Computational Biology \textbf{7}(2), e1001,066 (2011)

\bibitem{wang2022combining}
Wang, H., Leskovec, J.: Combining graph convolutional neural networks and label propagation.
\newblock ACM Transactions on Information Systems \textbf{40}(4), 1--27 (2022)

\bibitem{Wu2019ACS}
Wu, Z., Pan, S., Chen, F., Long, G., Zhang, C., Yu, P.S.: A comprehensive survey on graph neural networks.
\newblock IEEE Transactions on Neural Networks and Learning Systems \textbf{32}, 4--24 (2019)

\end{thebibliography}

\clearpage

\appendix

\begin{center}
{\large\bf Appendix}
\end{center}

\section{Proof}

\setcounter{figure}{0}
\setcounter{theorem}{0}
\renewcommand{\thealgorithm}{C\arabic{algorithm}}
\renewcommand{\thefigure}{E\arabic{figure}}
\renewcommand{\thesubsection}{\thesection.\arabic{subsection}}
\renewcommand{\thesubsubsection}{\thesubsection.\arabic{subsubsection}}
\pagenumbering{arabic}
\renewcommand{\thepage}{\arabic{page}}

\begin{theorem}
Given the random variable pair $(X,Y)$ of finite moments. Let $U \in \mathbb{R}^{K \times p}$, where each row satisfies
\begin{align*}
U(k,:) &= E(X|Y=k).
\end{align*}

Under the above probabilistic assumption for the sample data $(\mathbf{X}, \mathbf{Y})$, the matrix $\mathbf{U}$ in Section~\ref{sec2} satisfies:
\begin{align*}
\|\mathbf{U}- U \|_{F} \stackrel{n\rightarrow \infty}{\rightarrow} 0
\end{align*}
for the Frobenius matrix norm.
\end{theorem}
\begin{proof}
As 
\begin{align*}
\mathbf{U} = \mathbf{W}^{T} \mathbf{X},
\end{align*}
it follows from basic conditional probability and law of large number that
\begin{align*}
\mathbf{U}(k,s) &= \mathbf{W}(:,k)^{T} \mathbf{X}(:,s) \\
& = \frac{1}{n_k} \sum_{i=1}^{n} {\mathbf{X}(i,s)1(\mathbf{Y}_i=k)}\\
& \stackrel{n\rightarrow \infty}{\rightarrow} E(X_s|Y=k),
\end{align*}
where $X_s$ denotes the $s$th dimension of the random variable $X$.

Concatenating all dimensions for $s=1,2,\ldots,p$ and all $k=1,\ldots,K$, we can conclude that
\begin{align*}
\|\mathbf{U}- U \|_{F} \stackrel{n\rightarrow \infty}{\rightarrow} 0.
\end{align*}
\end{proof}

\begin{theorem}
Suppose $X$ is a mixture of normal distribution, with each normal component having the same variance matrix, i.e.,
\begin{align*}
X|(Y=y) \sim N(\mu_{y}, \Sigma).
\end{align*}
Then the embedding variable $Z=X U^{T}$ satisfies:
\begin{align*}
Z|(Y=y) \sim N(\mu_{y}U^{T}, U\Sigma U^{T}),
\end{align*}
in which case the linear discriminant analysis is the Bayes optimal classifier for $(Z,Y)$.
\end{theorem}
\begin{proof}
When the conditional distribution of $X|Y$ is normally distributed, the distribution of $Z|Y$ is also normally distributed because $U$ is a fixed matrix. Additionally, the mean and variance of $Z|Y=k$ are given by:
\begin{align*}
E(Z|Y=k)&=\mu_{k} U
\end{align*}
and 
\begin{align*}
Var(Z|Y=k)&=U \Sigma  U^{T} 
\end{align*}
In this case, the variance of $Z|Y=k$ is independent of the class label $k$, and as a result, linear discriminant analysis is the Bayes optimal classifier \citep{DevroyeGyorfiLugosiBook}.
\end{proof}

\begin{theorem}
Suppose that the conditional expectations in the original space $\mathbb{R}^{p}$ are separated by certain margin, i.e.
\begin{align*}
\|E(X | Y=k) - E(X | Y=j)\| = \alpha_{kj}
\end{align*}
for any $k \neq j$.

When using the inner product kernel, the resulting encoder embedding variable $Z$ in the reduced $\mathbb{R}^{K}$ space satisfies
\begin{align*}
\|E(Z | Y=k) - E(Z | Y=j)\| \geq \frac{1}{\sqrt{2}} \alpha_{kj}^2.
\end{align*}
\end{theorem}
\begin{proof}
Without loss of generality, it suffices to prove this for $k=1$ and $j=2$. Denoting $\mu_1=E(X | Y=1)$ and $\mu_2=E(X | Y=2)$, we have 
\begin{align*}
\|\mu_1 - \mu_2\|^2 &= \|\mu_1\|^2 + \|\mu_2\|^2 - 2 \mu_1 \mu_2^{T} = \alpha^2
\end{align*}
by the margin assumption on $X$. Next, using the basic property of the L2 norm, we can further deduce:
\begin{align*}
&\|E(Z | Y=1) - E(Z | Y=2)\|^2 \\
=& \sum_{j=1}^{K} (E(Z_j | Y=1) - E(Z_j | Y=2))^2\\
 \geq& \sum_{j=1}^{2} (E(Z_j | Y=1) - E(Z_j | Y=2))^2.
\end{align*}
In other words, to prove that $Z$ preserves the margin of separation in terms of the conditional expectation, it suffices to consider the first two dimensions $Z_1$ and $Z_2$, as the remaining dimensions cannot decrease the separation in L2 norm.

Under the inner product kernel, we can express the conditional expectations as follows:
\begin{align*}
E(Z_j | Y=1) = \mu_{1}U(j,:)^{T} &= \mu_{1} \mu_{j}^{T},\\
E(Z_j | Y=2) = \mu_{2}U(j,:)^{T} &= \mu_{2} \mu_{j}^{T}.
\end{align*}
From this, we can derive the following:
 \begin{align*}
 &\sum_{j=1}^{2} (E(Z_j | Y=1) - E(Z_j | Y=2))^2\\
 =&  (\mu_{1}\mu_{1}^{T}-\mu_{2}\mu_{1}^{T})^2 +  (\mu_{2}\mu_{1}^{T}-\mu_{2}\mu_{2}^{T})^2 \\
 =& \mu_{1}\mu_{1}^{T}\mu_{1}\mu_{1}^{T}-2\mu_{2}\mu_{1}^{T}\mu_{1}\mu_{1}^{T}+\mu_{2}\mu_{1}^{T}\mu_{2}\mu_{1}^{T} \\
&+\mu_{1}\mu_{2}^{T}\mu_{1}\mu_{2}^{T}-2\mu_{2}\mu_{2}^{T}\mu_{1}\mu_{2}^{T}+\mu_{2}\mu_{2}^{T}\mu_{2}\mu_{2}^{T} \\
=& \|\mu_{1}\|^4+\|\mu_{2}\|^4+2\mu_{1}\mu_{2}^{T}\mu_{1}\mu_{2}^{T}-2\mu_{1}\mu_{2}^{T}(\|\mu_{1}\|^2+\|\mu_{2}\|^2)\\
=& \|\mu_{1}\|^4+\|\mu_{2}\|^4+2\mu_{1}\mu_{2}^{T}\mu_{1}\mu_{2}^{T}-2\mu_{1}\mu_{2}^{T}(2 \mu_{1}\mu_{2}^{T} + \alpha^2)\\
=& \|\mu_{1}\|^4+\|\mu_{2}\|^4 - 2\mu_{1}\mu_{2}^{T}\mu_{1}\mu_{2}^{T}-2 \alpha^2 \mu_{1}\mu_{2}^{T}\\
=& \|\mu_{1}\|^4+\|\mu_{2}\|^4 - (\|\mu_1\|^2 + \|\mu_2\|^2-\alpha^2)^2/2\\
& -\alpha^2 (\|\mu_1\|^2 + \|\mu_2\|^2-\alpha^2)\\
=& \|\mu_{1}\|^4/2+\|\mu_{2}\|^4/2 - \|\mu_{1}\|^2\|\mu_{2}\|^2+\alpha^4/2\\
=& (\|\mu_{1}\|^2 -\|\mu_{2}\|^2)^2 /2  + \alpha^4/2\\
\geq & a^4/2.
\end{align*}
As a result, we can conclude that:
\begin{align*}
\|E(Z | Y=1) - E(Z | Y=2)\| \geq  \frac{1}{\sqrt{2}} \alpha^2.
\end{align*}

Finally, the proof presented here is applicable to any pair of classes $(k,j)$, not limited to $(k=1,j=2)$. To extend the proof to other class pairs, one can apply the same logic to $Z_j$ and $Z_k$, compute $(E(Z_j | Y=k) - E(Z_j | Y=l))^2 + (E(Z_k | Y=k) - E(Z_k | Y=j))^2$, and the inequality steps follow in the same manner.
\end{proof}

\end{document}